\documentclass{article}

\PassOptionsToPackage{round}{natbib}
\usepackage[final]{nips_2017}

\usepackage[utf8]{inputenc} % allow utf-8 input
\usepackage[T1]{fontenc}    % use 8-bit T1 fonts
\usepackage[hidelinks]{hyperref}       % hyperlinks
\usepackage{url}            % simple URL typesetting
\usepackage{booktabs}       % professional-quality tables
\usepackage{nicefrac}       % compact symbols for 1/2, etc.
\usepackage{microtype}      % microtypography
\usepackage{amsbsy}
\usepackage{footmisc}
\usepackage{amsmath,amsfonts,amsthm,amssymb}
\usepackage{multirow}
\usepackage{algorithm}
\usepackage{algorithmic}
\usepackage{graphicx}
\usepackage{thmtools}
\usepackage{thm-restate}
\usepackage{cleveref}
\usepackage{bm}

\newtheorem{theorem}{Theorem}
\newtheorem{proposition}{Proposition}

\DeclareMathOperator*{\argmin}{arg\,min}

\usepackage{color}

\definecolor{green}{rgb}{0.1,0.5,0.1}
\definecolor{red}{rgb}{1.0,0.2,0.2}

\newcommand{\add}[1]{{\color{green} #1}}

\title{Style Transfer from Non-Parallel Text by Cross-Alignment}

\author{
	Tianxiao Shen$^1$~~~~~Tao Lei$^2$~~~~~Regina Barzilay$^1$~~~~~Tommi Jaakkola$^1$\\
    $^1$MIT CSAIL~~~~~~~~~~$^2$ASAPP Inc.\\
    $^1$\texttt{\{tianxiao, regina, tommi\}@csail.mit.edu}~~~~$^2$\texttt{tao@asapp.com}
}

\begin{document}

\maketitle

\begin{abstract}
This paper focuses on style transfer on the basis of non-parallel text. This is an instance of a broad family of problems including machine translation, decipherment, and sentiment modification.
The key challenge is to separate the content from other aspects such as style.
We assume a shared latent content distribution across different text corpora, and propose a method that leverages refined alignment of latent representations to perform style transfer. The transferred sentences from one style should match example sentences from the other style as a population.
%We leverage refined alignment of latent representations across mono-lingual text corpora with different characteristics.
%We deliberately modify encoded examples according to their characteristics, requiring the reproduced instances to match available examples with the altered characteristics as a population.
We demonstrate the effectiveness of this cross-alignment method on three tasks: sentiment modification, decipherment of word substitution ciphers, and recovery of word order.\footnote{Our code and data are available at \url{https://github.com/shentianxiao/language-style-transfer}.}
\end{abstract}

\section{Introduction}
Using massive amounts of parallel data has been essential for recent advances in text generation tasks, such as machine translation and summarization.
However, in many text generation problems, we can only assume access to non-parallel or mono-lingual data. Problems such as decipherment or style transfer are all instances of this family of tasks. In all of these problems, we must preserve the content of the source sentence but render the sentence consistent with desired presentation constraints (e.g., style, plaintext/ciphertext). 

%A great deal of effort has gone into developing methods for summarizing %content in vector spaces and reconstituting the corresponding examples, %whether images, text, or molecules. More recently, refined approaches aim to %separate content from other conceivably modifiable characteristics such as %type, shape, style or language. The goal is to be able to reproduce objects %with the same content but differing in other characteristics such as style. An %interesting sub-class of this family of problems can be reasonably expected to %only involve non-parallel data such as decipherment, style changes of %sentences, and many others. We focus in particular on guided text generation. 

The goal of controlling one aspect of a sentence such as style independently of its content requires that we can disentangle the two. However, these aspects interact in subtle ways in natural language sentences, and we can succeed in this task only approximately even in the case of parallel data. Our task is more challenging here. We merely assume access to two corpora of sentences with the same distribution of content albeit rendered in different styles. Our goal is to demonstrate that this distributional equivalence of content, if exploited carefully, suffices for us to learn to map a sentence in one style to a style-independent content vector and then decode it to a sentence with the same content but a different style. 

%\add{This requires us to separate the content from style. However, the concepts of style and content of natural language are subtle and tangled, and in certain contexts their factorization is ambiguous. In this paper, we assume a shared latent content distribution across text corpora with different styles, and learn a mapping between a sentence and its style-content representation to capture these two degrees of freedom. With access only to non-parallel data, whether or not the sentences from different styles are paired correctly through the same content representation is . }
%With access only to non-parallel data, we must ensure that the disentangled representation really captures independent degrees of freedom with consistent effects on rendered sentences.
%A recent approach to this problem builds on variational auto-encoders (VAEs), dividing the latent representation into two or more parts, and enforces that the latent characteristics used to generate text can be reliably inferred back~\citep{hu2017controllable}. The decomposition of latent representation can be used to perform style transfer. However, in this way the quality of altered text is only partially controlled by the characteristics classifier.

In this paper, we introduce a refined alignment of sentence representations across text corpora. We learn an encoder that takes a sentence and its original style indicator as input, and maps it to a style-independent content representation. This is then passed to a  style-dependent decoder for rendering. We do not use typical VAEs for this mapping since it is imperative to keep the latent content representation rich and unperturbed. Indeed, richer latent content representations are much harder to align across the corpora and therefore they offer more informative content constraints. Moreover, we reap additional information from cross-generated (style-transferred) sentences, thereby getting two distributional alignment constraints. For example, positive sentences that are style-transferred into negative sentences should match, as a population, the given set of negative sentences. We illustrate this cross-alignment in Figure~\ref{fig:overview}.

\iffalse
%We dispense with VAEs that align posteriors to an already specified fixed distribution thus reaping little from additional cross-alignment.
We dispense with typical VAEs for this that predefines a latent distribution and aligns different corpora to it thus reaping little from additional cross-alignment.} Instead, we retain the complexity of the latent representation as opposed to how it is interpreted by the decoder. It leaves the transfer function relatively simple, which is helpful for content preservation.
%so as to maximize the information obtained from alignment.
Our theoretical results support this argument.
%In addition, we deliberately modify encoded examples (e.g., positive sentences, plaintext) according to their characteristics (altering them to be negative or ciphertext), and requiring that the resulting reproduced instances match, as a population, available examples with the analogous characteristics (negative, ciphertext).
\add{In addition, we intentionally alter the original style representation of sentences, and require that the resulting rendered sentences match available examples from the target style as a population.
}
\fi

\begin{figure}
\centering
\includegraphics[width=0.7\textwidth]{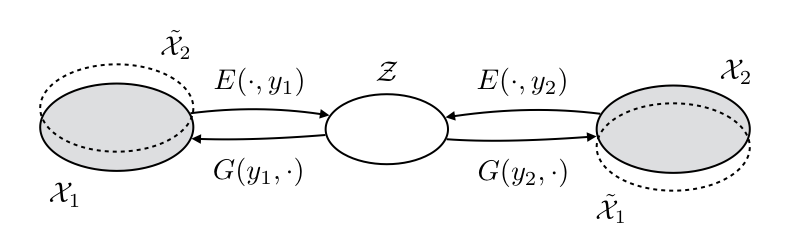}
\caption{
An overview of the proposed cross-alignment method. $\mathcal X_1$ and $\mathcal X_2$ are two sentence domains with different styles $y_1$ and $y_2$, and $\mathcal Z$ is the shared latent content space. Encoder $E$ maps a sentence to its content representation, and generator $G$ generates the sentence back when combining with the original style. When combining with a different style, transferred $\tilde{\mathcal X_1}$ is aligned with $\mathcal X_2$ and $\tilde{\mathcal X_2}$ is aligned with $\mathcal X_1$ at the distributional level.
}
\label{fig:overview}
\end{figure}

To demonstrate the flexibility of the proposed model, we evaluate it on three tasks: sentiment modification, decipherment of word substitution ciphers, and recovery of word order. In all of these applications, the model is trained on non-parallel data. On the sentiment modification task, the model successfully transfers the sentiment while keeps the content for 41.5\% of review sentences according to human evaluation,  compared to 41.0\% achieved by the control-gen model of~\cite{hu2017controllable}. It achieves strong performance on the decipherment and word order recovery tasks, reaching Bleu score of 57.4 and 26.1 respectively, obtaining 50.2 and 20.9 gap than a comparable method without cross-alignment.

\section{Related work}
% To Tianxiao: cite as much as you could
%  https://arxiv.org/pdf/1703.00848.pdf

\paragraph{Style transfer in vision} Non-parallel style transfer has been extensively studied in computer vision~\citep{gatys2016image,zhu2017unpaired,liu2016coupled,liu2017unsupervised,taigman2016unsupervised,kim2017learning,yi2017dualgan}. \citet{gatys2016image} explicitly extract content and style features, and then synthesize a new image by combining ``content'' features of one image with ``style'' features from another.
%For instance, \citet{gatys2016image} models content using deep features learned by a pre-trained neural network, and captures style by further processing these features to extract texture information.
More recent approaches learn generative networks directly via generative adversarial training~\citep{goodfellow2014generative} from two given data domains $\bm X_1$ and $\bm X_2$. The key computational challenge in this non-parallel setting is aligning the two domains. For example, CoupledGANs~\citep{liu2016coupled} employ weight-sharing between networks to learn cross-domain representation, whereas CycleGAN~\citep{zhu2017unpaired} introduces cycle consistency which relies on transitivity to regularize the transfer functions.  While our approach has a similar high-level architecture, the  discreteness of natural language does not allow us to reuse these models and necessitates the development of new methods.

\paragraph{Non-parallel transfer in natural language} In natural language processing, most tasks that involve generation (e.g., translation and 
summarization) are trained using parallel sentences. Our work most closely relates to approaches that do not utilize parallel data, but instead guide sentence generation from an indirect training signal~\citep{muller2017sequence,hu2017controllable}. For instance,~\citet{muller2017sequence} manipulate the hidden representation to generate sentences that satisfy a desired property (e.g., sentiment) as measured by a corresponding classifier. However, their model does not necessarily enforce content preservation.
More similar to our work,
\citet{hu2017controllable} aims at generating sentences with controllable attributes by learning disentangled latent representations~\citep{chen2016infoganir}.
Their model builds on variational auto-encoders (VAEs) and uses independency constraints to enforce that attributes can be reliably inferred back from generated sentences. While our model builds on distributional cross-alignment for the purpose of style transfer and content preservation, these constraints can be  added in the same way.

%More similar to our work, \citet{hu2017controllable} facilitate content preservation by explicitly modeling the style and content factorization. Using ideas from computer vision~\citep{chen2016infoganir}, this factorization is learned in an unsupervised manner.  The factorization algorithm of \citet{hu2017controllable} is primarily driven by the predictions of the property classifier. In contrast, our method utilizes adversarial training \citep{goodfellow2014generative} to guarantee distributional alignment of the common latent space driven by  content/style independence. By encouraging generated sentences to be close to the target domain, our model provides a richer guidance to overall sentence quality than a property classifier.

\paragraph{Adversarial training over discrete samples} Recently, a wide range of techniques addresses challenges associated with adversarial training over discrete samples generated by recurrent networks \citep{yu2016seqgan,lamb2016professor,hjelm2017boundary,che2017maximum}. In our work, we employ the Professor-Forcing algorithm \citep{lamb2016professor} which was originally proposed to close the gap between teacher-forcing during training and self-feeding during testing for recurrent networks. This design fits well with our scenario of style transfer that calls for cross-alignment.
By using continuous relaxation to approximate the discrete sampling process~\citep{jang2016categorical, maddison2016concrete}, the  training procedure can be effectively optimized through back-propagation~\citep{kusner2016gans,goyal2017differentiable}.

\section{Formulation}
In this section, we formalize the task of non-parallel style transfer and discuss the feasibility of the learning problem.
We assume the data are generated by the following process:
\begin{enumerate}
\item a latent style variable $\bm y$ is generated from some distribution $p(\bm y)$;
\item a latent content variable $\bm z$ is generated from some distribution $p(\bm z)$;
\item a datapoint $\bm x$ is generated from conditional distribution $p(\bm x|\bm y,\bm z)$. 
\end{enumerate}

We observe two datasets with the same content distribution but different styles $\bm y_1$ and $\bm y_2$, where $\bm y_1$ and $\bm y_2$ are unknown. Specifically, the two observed datasets $\bm X_1=\{\bm x_1^{(1)}, \cdots, \bm x_1^{(n)}\}$ and $\bm X_2=\{\bm x_2^{(1)}, \cdots, \bm x_2^{(m)}\}$ consist of samples drawn from $p(\bm x_1|\bm y_1)$ and $p(\bm x_2|\bm y_2)$ respectively. 
%let $\bm x_1$ be a random variable whose probability density function is $p(\bm x_1;\bm y_1)=p(\bm x_1|\bm y_1)$, and $\bm x_2$ be a random variable whose probability density function is $p(\bm x_2;\bm y_2)=p(\bm x_2|\bm y_2)$. The observed datasets $\bm X_1=\{\bm x_1^{(1)}, \cdots, \bm x_1^{(n)}\}$ and $\bm X_2=\{\bm x_2^{(1)}, \cdots, \bm x_2^{(m)}\}$ consist of samples from $\bm x_1$ and $\bm x_2$ respectively. 
We want to estimate the style transfer functions between them, namely $p(\bm x_1|\bm x_2;\bm y_1,\bm y_2)$ and $p(\bm x_2|\bm x_1;\bm y_1,\bm y_2)$. 
%\add{In other words, we want to infer the conditional or joint probability of $x_1$ and $x_2$ given only marginal distributions.}\\

%\paragraph{Necessary conditions.}
%The data distribution must have certain properties to enable the estimation. 
A question we must address is when this estimation problem is feasible.
Essentially, we only observe the marginal distributions of $\bm x_1$ and $\bm x_2$, yet we are going to recover their joint distribution:
\begin{align}
p(\bm x_1,\bm x_2|\bm y_1,\bm y_2)=\int_{\bm z}p(\bm z)p(\bm x_1|\bm y_1,\bm z)p(\bm x_2|\bm y_2,\bm z)d\bm z
\end{align}
As we only observe $p(\bm x_1|\bm y_1)$ and $p(\bm x_2|\bm y_2)$, $\bm y_1$ and $\bm y_2$ are unknown to us. If two different $\bm y$ and $\bm y'$ lead to the same distribution $p(\bm x|\bm y)=p(\bm x|\bm y')$, then given a dataset $\bm X$ sampled from it, its underlying style can be either $\bm y$ or $\bm y'$. Consider the following two cases: (1) both datasets $\bm X_1$ and $\bm X_2$ are sampled from the same style $\bm y$; (2) $\bm X_1$ and $\bm X_2$ are sampled from style $\bm y$ and $\bm y'$ respectively. These two scenarios have different joint distributions, but the observed marginal distributions are the same. To prevent such confusion, we constrain the underlying distributions as stated in the following proposition:
%This is possible only when their respective distributions and the transfer functions between them are constrained.\\

\begin{proposition}\label{prop:recover}
In the generative framework above, $\bm x_1$ and $\bm x_2$'s joint distribution can be recovered from their marginals only if for any different $\bm y,\bm y'\in \mathcal Y$, distributions $p(\bm x|\bm y)$ and $p(\bm x|\bm y')$ are different.
\end{proposition}

%Conversely, if different $\bm y$ leads to different distribution $p(\bm x|\bm y)$, given $\bm X_1$ sampled from $p(\bm x|\bm y_1)$ and $\bm X_2$ sampled from $p(\bm x|\bm y_2)$, their underlying styles $\bm y_1$ and $\bm y_2$ can be determined, and hence their joint distribution can be recovered.

This proposition basically says that $\bm X$ generated from different styles should be ``distinct'' enough, otherwise the transfer task between styles is not well defined. 
While this seems trivial, it may not hold even for simplified data distributions. The following examples illustrate how the transfer (and recovery) becomes feasible or infeasible under different model assumptions. As we shall see, for a certain family of styles $\mathcal Y$, the more complex distribution for $\bm z$, the more probable it is to recover the transfer function and the easier it is to search for the transfer.

\subsection{Example 1: Gaussian}
Consider the common choice that $\bm z\sim \mathcal N(\bm 0, \bm I)$ has a centered isotropic Gaussian distribution. Suppose a style $\bm y=(\bm A,\bm b)$ is an affine transformation, i.e. $\bm x=\bm A\bm z+\bm b+\bm \epsilon$, where $\bm\epsilon$ is a noise variable. For $\bm b=\bm 0$ and any orthogonal matrix $\bm A$, $\bm A\bm z+\bm b\sim N(\bm 0, \bm I)$ and hence $\bm x$ has the same distribution for any such styles $\bm y=(\bm A, \bm 0)$. In this case, the effect of rotation cannot be recovered. 

Interestingly, if $\bm z$ has \emph{\bf a more complex distribution}, such as a Gaussian mixture, then affine transformations can be uniquely determined.

\begin{restatable}{lemma}{gaussian}\label{lemma:gaussian}
Let $\bm z$ be a mixture of Gaussians $p(\bm z)=\sum_{k=1}^K \pi_k\mathcal N(\bm z;\bm\mu_k, \bm \Sigma_k)$. Assume $K\ge 2$, and there are two different $\bm\Sigma_i\ne\bm\Sigma_j$. Let $\mathcal Y=\{(\bm A,\bm b)||\bm A|\ne 0\}$ be all invertible affine transformations, and $p(\bm x|\bm y,\bm z)=\mathcal N(\bm x;\bm A\bm z+\bm b,\epsilon^2\bm I)$, in which $\epsilon$ is a noise. Then for all $\bm y\ne\bm y'\in\mathcal Y$, $p(\bm x|\bm y)$ and $p(\bm x|\bm y')$ are different distributions.
\end{restatable}

\begin{theorem}
If the distribution of $\bm z$ is a mixture of Gaussians which has more than two different components, and $\bm x_1, \bm x_2$ are two affine transformations of $\bm z$, then the transfer between them can be recovered given their respective marginals.
\end{theorem}

\subsection{Example 2: Word substitution}
Consider here another example when $\bm z$ is a bi-gram language model and a style $\bm y$ is a vocabulary in use that maps each ``content word'' onto its surface form (lexical form). If we observe two realizations $\bm x_1$ and $\bm x_2$ of the same language $\bm z$, the transfer and recovery problem becomes inferring a word alignment between $\bm x_1$ and $\bm x_2$. 

Note that this is a simplified version of language decipherment or translation. Nevertheless, the recovery problem is still sufficiently hard. To see this, let $\bm M_1, \bm M_2 \in \mathcal{R}^{n\times n}$ be the estimated bi-gram probability matrix of data $\bm X_1$ and $\bm X_2$ respectively. Seeking the word alignment is equivalent to finding a permutation matrix $\bm P$ such that $\bm P^\top \bm M_1 \bm P \approx \bm M_2$, which can be expressed as an optimization problem,
\begin{align*}
\min_{\bm P} ~\| \bm P^\top \bm M_1 \bm P - \bm M_2 \|^2
\end{align*}
The same formulation applies to graph isomorphism (GI) problems given $\bm M_1$ and $\bm M_2$ as the adjacency matrices of two graphs, suggesting that determining the existence and uniqueness of $\bm P$ is at least GI hard. Fortunately, if $\bm M$ as a graph is complex enough, the search problem could be more tractable. For instance, if each vertex's weights of incident edges as a set is unique, then finding the isomorphism can be done by simply matching the sets of edges. This assumption largely applies to our scenario where $\bm z$ is a complex language model. We empirically demonstrate this in the results section.

%\paragraph{Main message.} 
The above examples suggest that $\bm z$ as the latent content variable should carry most complexity of data $\bm x$, while $\bm y$ as the latent style variable should have relatively simple effects.
We construct the model accordingly in the next section.
%If $\bm y$ can easily interfere with $\bm z$ and lead to the same distribution of $\bm x$, then style transfer may involve content change and this is unrecoverable.

%\section{Language Style Transfer}
\section{Method}
Learning the style transfer function under our generative assumption is essentially learning the conditional distribution $p(\bm x_1|\bm x_2;\bm y_1,\bm y_2)$ and $p(\bm x_2|\bm x_1; \bm y_1,\bm y_2)$. Unlike in vision where images are continuous and hence the transfer functions can be learned and optimized directly, the discreteness of language requires us to operate through the latent space. Since $\bm x_1$ and $\bm x_2$ are conditionally independent given the latent content variable $\bm z$, 
\begin{align}
\begin{split}
p(\bm x_1|\bm x_2;\bm y_1,\bm y_2) &= \int_{\bm z} p(\bm x_1,\bm z|\bm x_2;\bm y_1,\bm y_2)d\bm z\\
&=\int_{\bm z} p(\bm z|\bm x_2,\bm y_2)\cdot p(\bm x_1|\bm y_1,\bm z)d\bm z\\
&= \mathbb E_{\bm z\sim p(\bm z|\bm x_2,\bm y_2)}[p(\bm x_1|\bm y_1,\bm z)]
\end{split}
\end{align}

This suggests us learning an auto-encoder model.
Specifically, a style transfer from $\bm x_2$ to $\bm x_1$ involves two steps---an encoding step that infers $\bm x_2$'s content $\bm z\sim p(\bm z|\bm x_2,\bm y_2)$, and a decoding step which generates the transferred counterpart from $p(\bm x_1|\bm y_1,\bm z)$. 
In this work, we approximate and train $p(\bm z|\bm x, \bm y)$ and $p(\bm x|\bm y, \bm z)$ using neural networks (where $\bm y\in\{\bm y_1, \bm y_2\}$).

Let $E:\mathcal X\times \mathcal Y\rightarrow \mathcal Z$ be an encoder that infers the content $\bm z$ for a given sentence $\bm x$ and a style $\bm y$%\footnote{In practice, $\bm y$ is simply an input vector to the neural encoder and generator that signals which style is used.}
, and $G:\mathcal Y\times \mathcal Z\rightarrow \mathcal X$ be a generator that generates a sentence $\bm x$ from a given style $\bm y$ and content $\bm z$. $E$ and $G$ form an auto-encoder when applying to the same style, and thus we have reconstruction loss,
\begin{align}\label{eq:rec}
\begin{split}
\mathcal L_{\text{rec}}(\bm\theta_E,\bm\theta_G)=~&\mathbb E_{\bm x_1\sim\bm X_1}[-\log p_G(\bm x_1|\bm y_1, E(\bm x_1,\bm y_1))]~+\\
&\mathbb E_{\bm x_2\sim\bm X_2}[-\log p_G(\bm x_2|\bm y_2, E(\bm x_2,\bm y_2))]
\end{split}
\end{align}
where $\bm\theta$ are the parameters to estimate.

In order to make a meaningful transfer by flipping the style, $\bm X_1$ and $\bm X_2$'s content space must coincide, as our generative framework presumed. To constrain that $\bm x_1$ and $\bm x_2$ are generated from the same latent content distribution $p(\bm z)$, one option is to apply a variational auto-encoder~\citep{kingma2013auto}. A VAE imposes a prior density $p(\bm z)$, such as $\bm z\sim\mathcal N(\bm 0,\bm I)$, and uses a KL-divergence regularizer to align both posteriors $p_E(\bm z|\bm x_1,\bm y_1)$ and $p_E(\bm z|\bm x_2,\bm y_2)$ to it,
\begin{align}\label{eq:kl}
\begin{split}
\mathcal L_{\text{KL}}(\bm\theta_E)=~&\mathbb E_{\bm x_1\sim\bm X_1}[D_{\text{KL}}(p_E(\bm z|\bm x_1,\bm y_1)\|p(\bm z))]+
\mathbb E_{\bm x_2\sim\bm X_2}[D_{\text{KL}}(p_E(\bm z|\bm x_2,\bm y_2)\|p(\bm z))]
\end{split}
\end{align}
The overall objective is to minimize $\mathcal L_{\text{rec}}+\mathcal L_{\text{KL}}$, whose opposite is the variational lower bound of data likelihood.

However, as we have argued in the previous section, restricting $\bm z$ to a simple and even distribution and pushing most complexity to the decoder may not be a good strategy for non-parallel style transfer. In contrast, a standard auto-encoder simply minimizes the reconstruction error, encouraging $\bm z$ to carry as much information about $\bm x$ as possible. On the other hand, it lowers the entropy in $p(\bm x|\bm y,\bm z)$, which helps to produce meaningful style transfer in practice as we flip between $\bm y_1$ and $\bm y_2$. Without explicitly modeling $p(\bm z)$, 
it is still possible to force  distributional alignment of $p(\bm z|\bm y_1)$ and $p(\bm z|\bm y_2)$.
To this end, we introduce two constrained variants of auto-encoder.

\subsection{Aligned auto-encoder}
%The first variant enforces both $\bm z_1=E(\bm x_1,\bm y_1)$ and $\bm z_2=E(\bm x_2,\bm y_2)$ to have the same distribution (where $\bm x_1\sim\bm X_1$ and $\bm x_2\sim \bm X_2$).
%To see this, consider the stochastic version of this encoding process in which $\bm x\sim p(\bm x|\bm y)$ and $\bm z\sim p(\bm z|\bm x,\bm y)$. The distribution of $\bm z$ through this process is $\int_{\bm x} p(\bm x|\bm y) p(\bm z|\bm x, \bm y) d\bm x = p(\bm z|\bm y)$.
%Note since $\bm y$ and $\bm z$ are independent, $p(\bm z|\bm y_1)=p(\bm z|\bm y_2)=p(\bm z)$.
%Although we do not constrain $\bm z$ to any particular distribution, for $\bm x_1\sim\bm X_1$ and $\bm x_2\sim\bm X_2$, $\bm z_1=E(\bm x_1,\bm y_1)$ and $\bm z_2=E(\bm x_2,\bm y_2)$ should have the same distribution, namely $p(\bm z)$. Formally, for any style $\bm y$, if we first sample $\bm x\sim p(\bm x|\bm y)$, then sample $\bm z\sim p(\bm z|\bm x,\bm y)$, we are actually sampling $\bm x,\bm z\sim p(\bm x,\bm z|\bm y)$. Hence $\bm z$ has distribution $p(\bm z|\bm y)=p(\bm z)$, since $\bm z$ and $\bm y$ are independent.

%Following this motivation, we revise the reconstruction objective~(\ref{eq:rec}) as a constrained optimization problem:

Dispense with VAEs that make an explicit assumption about $p(\bm z)$ and align both posteriors to it, we align $p_E(\bm z|\bm y_1)$ and $p_E(\bm z|\bm y_2)$ with each other, which leads to the following constrained optimization problem:
\begin{align}
\begin{split}
\bm\theta^*= \argmin_{\bm\theta}&~ \mathcal L_{\text{rec}}(\bm\theta_E,\bm\theta_G)\\
\text{s.t.}\quad E(\bm x_1,\bm y_1) \stackrel{\text{d}}{=} E(\bm x_2,\bm y_2) &\qquad \bm x_1\sim\bm X_1,\bm x_2\sim\bm X_2
\end{split}
\end{align}
In practice, a Lagrangian relaxation of the primal problem is instead optimized.
We introduce an adversarial discriminator $D$ to align the aggregated posterior distribution of $\bm z$ from different styles~\citep{makhzani2015adversarial}.
$D$ aims to distinguish between these two distributions:
\begin{align}
\begin{split}
\mathcal L_{\text{adv}}(\bm\theta_E,\bm\theta_D)=~\mathbb E_{\bm x_1\sim\bm X_1}[-\log D(E(\bm x_1,\bm y_1))]~+
\mathbb E_{\bm x_2\sim\bm X_2}[-\log(1-D(E(\bm x_2,\bm y_2)))]
\end{split}
\end{align}

The overall training objective is a min-max game played among the encoder $E$, generator $G$ and discriminator $D$. They constitute an aligned auto-encoder:
\begin{align}
\min_{E,G}\max_D \mathcal L_{\text{rec}}-\lambda\mathcal L_{\text{adv}}
\end{align}

%Comparing to variational auto-encoder, where the objective function is the reconstruction error plus KL divergence regularizer $D_{KL}(p_E(\bm z|\bm x,\bm y)\|p(\bm z))$, they align both posteriors $p_E(\bm z|\bm x_1,\bm y_1)$ and $p_E(\bm z|\bm x_2,\bm y_2)$ to the assumed prior $p(\bm z)$. Here we align $p_E(\bm z|\bm y_1)$ and $p_E(\bm z|\bm y_2)$ with each other and do not make any explicit assumptions about $p(\bm z)$.

We implement the encoder $E$ and generator $G$ using single-layer RNNs with GRU cell. $E$ takes an input sentence $\bm x$ with initial hidden state $\bm y$, and outputs the last hidden state $\bm z$ as its content representation. $G$ generates a sentence $\bm x$ conditioned on latent state $(\bm y,\bm z)$.
%Following standard practice, we implement the encoder $E:\mathcal X\times \mathcal Y\rightarrow \mathcal Z$ using a single-layer RNN that takes an input sentence $\bm x$ with initial hidden state $\bm y$, and outputs the last hidden state $\bm z$ as its content representation. The generator $G:\mathcal Y\times \mathcal Z\rightarrow \mathcal X$ is an RNN that specifies a sentence distribution conditioned on latent state $(\bm y,\bm z)$.
To align the distributions of $\bm z_1=E(\bm x_1,\bm y_1)$ and $\bm z_2=E(\bm x_2,\bm y_2)$, the discriminator $D$ is a feed-forward network with a single hidden layer and a sigmoid output layer.

\subsection{Cross-aligned auto-encoder}
The second variant, cross-aligned auto-encoder, directly aligns the transfered samples from one style with the true samples from the other. 
Under the generative assumption, $p(\bm x_2| \bm y_2) = \int_{\bm x_1} p(\bm x_2|\bm x_1;\bm y_1,\bm y_2) p(\bm x_1|\bm y_1)d\bm x_1$, thus $\bm x_2$ (sampled from the left-hand side) should exhibit the same distribution as transferred $\bm x_1$ (sampled from the right-hand side), and vice versa. Similar to our first model, the second model uses two discriminators $D_1$ and $D_2$ to align the populations. $D_1$'s job is to distinguish between real $\bm x_1$ and transferred $\bm x_2$, and $D_2$'s job is to distinguish between real $\bm x_2$ and transferred $\bm x_1$.
%In theory, if $\bm z_1$ and $\bm z_2$ are perfectly aligned $\mathcal Z_1=\mathcal Z_2=\mathcal Z$, and through the reconstruction loss the generator $G$ learns to map $\bm y_1\times \mathcal Z$ to $\mathcal X_1$ and map $\bm y_2\times \mathcal Z$ to $\mathcal X_2$, then our transfer process---first get the content $\bm z$, then flip the style $\bm y$, and generate a new sentence $\tilde{\bm x}$---is perfectly fine. In practice, however, when $\bm z_1$ and $\bm z_2$ are not perfectly aligned, their misalignment will accumulate through the recurrent network of generator $G$, and the transferred sentence may be far from the target domain. This problem is even more prominent here because we do not impose a prior for $\bm z$ and thus outliers are not punished by their distance. 
%In order to strengthen the alignment constraint, we introduce two discriminators $D_1$ and $D_2$ to regulate the recurrent generating process. $D_1$'s job is to distinguish between real $\bm x_1$ and transferred $\bm x_2$, and $D_2$'s job is to distinguish between real $\bm x_2$ and transferred $\bm x_1$.

Adversarial training over the discrete samples generated by $G$ hinders gradients propagation.
Although sampling-based gradient estimator such as REINFORCE~\citep{williams1992simple} can by adopted, training with these methods can be unstable due to the high variance of the sampled gradient.
Instead, we employ two recent techniques to approximate the discrete training~\citep{hu2017controllable,lamb2016professor}.
First, instead of feeding a single sampled word as the input to the generator RNN, we use the softmax distribution over words instead.
Specifically, during the generating process of transferred $\bm x_2$ from $G(\bm y_1,\bm z_2)$, suppose at time step $t$ the output logit vector is $\bm v_t$. We feed its peaked distribution $\text{softmax}(\bm v_t/\gamma)$ as the next input, where $\gamma \in (0,1)$ is a temperature parameter.

%If we pass a real sentence $\bm x_1$ and a soft generated sentence $G(\bm y_1,\bm z_2)$ to the discriminator, it may easily distinguish them by the hard/soft difference, and thus do not provide any useful information about the quality of transferred sentence. 
Secondly, we use Professor-Forcing \citep{lamb2016professor} to match the sequence of hidden states instead of the output words, which contains the information about outputs and is smoothly distributed. That is, the input to the discriminator $D_1$ is the sequence of hidden states of either (1) $G(\bm y_1,\bm z_1)$ teacher-forced by a real example $\bm x_1$, or (2) $G(\bm y_1,\bm z_2)$ self-fed by previous soft distributions. 

The running procedure of our cross-aligned auto-encoder is illustrated in Figure~\ref{fig:model}. Note that cross-aligning strengthens the alignment of latent variable $\bm z$ over the recurrent network of generator $G$. By aligning the whole sequence of hidden states, it prevents $\bm z_1$ and $\bm z_2$'s initial misalignment from propagating through the recurrent generating process, as a result of which the transferred sentence may end up somewhere far from the target domain.

We implement both $D_1$ and $D_2$ using convolutional neural networks for sequence classification~\citep{kim2014convolutional}. The training algorithm is presented in Algorithm~\ref{alg:train}.

\begin{figure}
\centering
\includegraphics[width=0.7\textwidth]{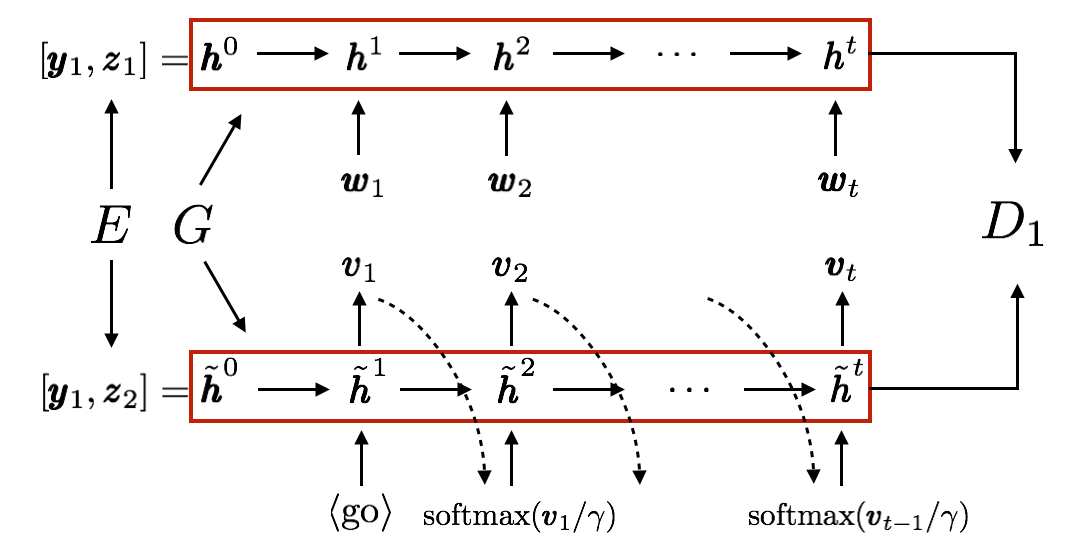}
\caption{Cross-aligning between $\bm x_1$ and transferred $\bm x_2$. For $\bm x_1$, $G$ is teacher-forced by its words $\bm w_1\bm w_2\cdots\bm w_t$. For transfered $\bm x_2$, $G$ is self-fed by previous output logits. The sequence of hidden states $\bm h^0,\cdots,\bm h^t$ and $\tilde{\bm h}^0,\cdots,\tilde{\bm h}^t$ are passed to discriminator $D_1$ to be aligned.
Note that our first variant aligned auto-encoder is a special case of this, where only $\bm h^0$ and $\tilde{\bm h}^0$, i.e. $\bm z_1$ and $\bm z_2$, are aligned.
}
\label{fig:model}
\end{figure}

\begin{algorithm}[!t!]
\caption{Cross-aligned auto-encoder training. The hyper-parameters are set as $\lambda=1,\gamma=0.001$ and learning rate is $0.0001$ for all experiments in this paper.}
\label{alg:train}
\begin{algorithmic}
\REQUIRE Two corpora of different styles $\bm X_1,\bm X_2$. Lagrange multiplier $\lambda$, temperature $\gamma$.
\STATE Initialize $\bm\theta_E,\bm\theta_G,\bm\theta_{D_1},\bm\theta_{D_2}$
\REPEAT
\FOR{$p=1,2$; $q=2,1$}
  \STATE Sample a mini-batch of $k$ examples $\{\bm x_p^{(i)}\}_{i=1}^k$ from $\bm X_p$
  \STATE Get the latent content representations $\bm z_p^{(i)}=E(\bm x_p^{(i)},\bm y_p)$
  \STATE Unroll $G$ from initial state $(\bm y_p,\bm z_p^{(i)})$ by feeding $\bm x_p^{(i)}$, and get the hidden states sequence $\bm h_p^{(i)}$
  \STATE Unroll $G$ from initial state $(\bm y_q,\bm z_p^{(i)})$ by feeding previous soft output distribution with temperature $\gamma$, and get the transferred hidden states sequence $\tilde{\bm h}_p^{(i)}$
\ENDFOR
\STATE Compute the reconstruction $\mathcal L_{\text{rec}}$ by Eq. (\ref{eq:rec})
\STATE Compute $D_1$'s (and symmetrically $D_2$'s) loss:
\begin{align}
\mathcal L_{\text{adv}_1}=-\frac{1}{k}\sum_{i=1}^k \log D_1(\bm h_1^{(i)})-\frac{1}{k}\sum_{i=1}^k \log (1-D_1(\tilde{\bm h}_2^{(i)}))
\end{align}
\STATE Update $\{\bm\theta_E,\bm\theta_G\}$ by gradient descent on loss
\begin{align}
\mathcal L_{\text{rec}}-\lambda(\mathcal L_{\text{adv}_1}+\mathcal L_{\text{adv}_2})
\end{align}
\STATE Update $\bm\theta_{D_1}$ and $\bm\theta_{D_2}$ by gradient descent on loss $\mathcal L_{\text{adv}_1}$ and $\mathcal L_{\text{adv}_2}$ respectively
\UNTIL{convergence}
\ENSURE Style transfer functions $G(\bm y_2, E(\cdot,\bm y_1)):\mathcal X_1\rightarrow \mathcal X_2$ and $G(\bm y_1, E(\cdot,\bm y_2)):\mathcal X_2\rightarrow \mathcal X_1$
\end{algorithmic}
\end{algorithm}

\section{Experimental setup}

\paragraph{Sentiment modification} Our first experiment focuses on text rewriting with the goal of changing the underlying sentiment, which can be regarded as ``style transfer'' between negative and positive sentences. We run experiments on Yelp restaurant reviews, utilizing readily available user ratings associated with each review. Following standard practice, reviews with rating above three are considered positive, and those below three are considered negative.
While our model operates at the sentence level, the sentiment annotations in our dataset are provided at the document level. We assume that all the sentences in a document have the same sentiment. This is clearly an oversimplification, since some sentences (e.g., background) are sentiment neutral. Given that such sentences are more common in long reviews, we filter out reviews that exceed 10 sentences. We further filter the remaining sentences by eliminating those that exceed 15 words. The resulting dataset has 250K negative sentences, and 350K positive ones. The vocabulary size is 10K after replacing words occurring less than 5 times with the ``<unk>'' token. As a baseline model, we compare against 
the control-gen model of \citet{hu2017controllable}. 

To quantitatively evaluate the transfered sentences, we adopt a model-based evaluation metric similar to the one used for image transfer~\citep{isola2016image}. Specifically, we measure how often a transferred sentence has the correct sentiment according to a pre-trained sentiment classifier. For this purpose, we use the TextCNN model as described in~\citet{kim2014convolutional}. On our simplified dataset for style transfer, it achieves nearly perfect accuracy of 97.4\%. 

While the quantitative evaluation provides some indication of transfer quality, it does not capture all the aspects  
of this generation task. Therefore, we also perform two human evaluations on 500 sentences randomly selected from the test set\footnote{we eliminated 37 sentences from them that were judged as neutral by human judges.}. In the first evaluation, the judges were asked to rank generated sentences in terms of their fluency and sentiment. Fluency was rated from 1 (unreadable) to 4 (perfect), while sentiment categories were ``positive'', ``negative'', or ``neither'' (which could be contradictory, neutral or nonsensical). In the second evaluation, we evaluate the transfer process comparatively. The annotator was shown a source sentence and the corresponding outputs of the systems in a random order, and was asked ``Which transferred sentence is semantically equivalent to the source sentence with an opposite sentiment?''. They can be both satisfactory, A/B is better, or both unsatisfactory. We collect two labels for each question. The label agreement and conflict resolution strategy can be found in the supplementary material. Note that the two evaluations are not redundant. For instance, a system that always generates the same grammatically correct sentence with the right sentiment independently of the source sentence will score high in the first evaluation setup, but low in the second one.

%In these experiments, we compare the two variants of our model with variational autoencoder.

\paragraph{Word substitution decipherment} Our second set of experiments involves decipherment of word substitution ciphers, which has been previously explored in NLP literature \citep{dou2012large, nuhn2013decipherment}. These ciphers replace every word in plaintext (natural language) with a cipher token according to a 1-to-1 substitution key. The decipherment task is to recover the plaintext from ciphertext. It is trivial if we have access to parallel data. However we are interested to consider a non-parallel decipherment scenario. For training, we select 200K sentences as $\bm X_1$, and apply a substitution cipher $f$ on a different set of 200K sentences to get $\bm X_2$. While these sentences are non-parallel, they are drawn from the same distribution from the review dataset. The development and test sets have 100K parallel sentences $\bm D_1=\{\bm x^{(1)},\cdots,\bm x^{(n)}\}$ and $\bm D_2=\{f(\bm x^{(1)}),\cdots,f(\bm x^{(n)})\}$. We can quantitatively compare between $\bm D_1$ and transferred (deciphered) $\bm D_2$ using Bleu score~\citep{papineni2002bleu}.

Clearly, the difficulty of this decipherment task depends on the number of substituted words. Therefore, we report model performance with respect to the percentage of the substituted vocabulary. Note that the transfer models do not know that $f$ is a word substitution function. They learn it entirely from the data distribution.

In addition to having different transfer models, we introduce a simple decipherment baseline based on word frequency. Specifically, we assume that words shared between $\bm X_1$ and $\bm X_2$  do not require translation. The rest of the words are mapped based on their frequency, and ties are broken arbitrarily. Finally, to assess the difficulty of the task, we report the 
accuracy of a machine translation system trained on a parallel corpus~\citep{klein2017opennmt}.

\paragraph{Word order recovery}
Our final experiments focus on the word ordering task, also known as bag translation~\citep{brown1990statistical,Schmaltz2016WordOW}. By learning the style transfer functions between original English sentences $\bm X_1$ and shuffled English sentences $\bm X_2$, the model can be used to recover the original word order of a shuffled sentence (or conversely to randomly permute a sentence). The process to construct non-parallel training data and parallel testing data is the same as in the word substitution decipherment experiment. Again the transfer models do not know that $f$ is a shuffle function and learn it completely from data.
%The model is provided with a random permutation of a sentence, and it has to retrieve its grammatical ordering. As in the case of word substitution ciphers, the training consists of non-parallel sentences. The overall process for constructing the training and testing data follows the process described above.

%Tianxiao: do we have a special baseline for this one?

\section{Results}

\paragraph{Sentiment modification}
Table~\ref{tab:sentiment} and Table~\ref{tab:human} show the performance of various models for both human and automatic evaluation. The control-gen model of \citet{hu2017controllable} performs better in terms of sentiment accuracy in both evaluations. This is not surprising as their generation is directly guided by a sentiment classifier. Their system also achieves higher fluency score. However, these gains do not translate into improvements in terms of the overall transfer, where our model faired better. As can be seen from the examples listed in Table~\ref{tab:examples}, our model is more consistent with  the grammatical structure and semantic meaning of the source sentence. In contrast, their model achieves sentiment change by generating an entirely new sentence which has little overlap with the source. The discrepancy between the two experiments demonstrates the crucial importance of developing appropriate evaluation measures to compare models for style transfer. 

%The cross-aligned auto-encoder achieves 78\% accuracy, outperforming other models by a substantial margin.
%The low performance of variational auto-encoder (23\%) clearly demonstrates that it is not suitable for this transfer task. 

\iffalse
Table~\ref{tab:examples} shows a few samples of sentiment transfer generated by our cross-aligned auto-encoder. Along with showing high-quality rewritings, the table reveals several examples of erroneous transformations.
Our manual analysis of the data reveals that such mistakes commonly fall into two classes: content modification and grammaticality. The fourth pair
in Table~\ref{tab:examples} is an instance of the first class -- while it successfully makes the sentence positive, it changes the topic from Mexican to Italian food. The second class is exemplified by 
the very last pair in Table~\ref{tab:examples}, where the rewritten sentence contains the grammatically incorrect  phrase "rushed with a couple of work".
\fi

\begin{table}
\begin{center}
\begin{tabular}{l|c}
\toprule
Method & accuracy\\
\toprule
\citet{hu2017controllable} & 83.5\\
\hline
Variational auto-encoder & 23.2\\
Aligned auto-encoder & 48.3\\
Cross-aligned auto-encoder & 78.4\\
\toprule
\end{tabular}
\end{center}
\caption{Sentiment accuracy of transferred sentences, as measured by a pretrained classifier.}\label{tab:sentiment}
\end{table}

\begin{table}
\begin{center}
\begin{tabular}{l|cc|c}
\toprule
Method & sentiment & fluency & overall transfer\\
\toprule
\citet{hu2017controllable} & 70.8 & 3.2 & 41.0\\
Cross-align & 62.6 & 2.8 & 41.5\\
\toprule
\end{tabular}
\end{center}
\caption{Human evaluations on sentiment, fluency and overall transfer quality. Fluency rating is from 1 (unreadable) to 4 (perfect). Overall transfer quality is evaluated in a comparative manner, where the judge is shown a source sentence and two transferred sentences, and decides whether they are both good, both bad, or one is better.}\label{tab:human}
\end{table}

\iffalse
\begin{table}
\begin{center}
\begin{tabular}{@{~~~~~}l@{~~~~~}}
\toprule
\multicolumn{1}{c}{Sentiment transfer from negative to positive}\\
\toprule
i would recommend find another place .\\
i would recommend this place again !\\[7pt]
do not like it at all !\\
all in all, it 's great !\\[7pt]
i regret not having the time to shop around .\\
i have a great experience here .\\[7pt]
average mexican food .\\
authentic italian food .\\
%\hline
%If we were to get this server again, I'd ask to move.\\
%If you gotta have the chance and we had to go back.\\
\toprule
\multicolumn{1}{c}{Sentiment transfer from positive to negative}\\
\toprule
really good food that is fast and healthy .\\
really bland and bad , and terrible .\\[7pt]
you will notice that i have given this restaurant five stars .\\
you should give this place zero stars .\\[7pt]
definitely a place you can bring the family or just go for happy hour !\\
do not waste of your money, go somewhere else !\\[7pt]
our waitress was very friendly and checked up on us a couple of times .\\
our waitress was very rude and rushed with a couple of work .\\
\toprule
\end{tabular}
\end{center}
\caption{Samples from our cross-aligned auto-encoder. The first line is an input sentence, and the second line is the generated sentence after sentiment transfer.}\label{tab:examples}
\end{table}
\fi

\begin{table}
\begin{center}
\begin{tabular}{@{~~~~~}l@{~~~~~}}
\toprule
\multicolumn{1}{c}{From negative to positive}\\
\toprule
consistently slow .\\
consistently good .\\
consistently fast .\\[7pt]

%every time i go to this location it 's slow !\\
%this place is good to work with every other day !\\
%every time i go to this location it 's fast !\\[7pt]

%the desserts were very bland .\\
%the desserts were very good .\\
%the desserts were excellent .\\[7pt]

my goodness it was so gross .\\
my husband 's steak was phenomenal .\\
my goodness was so awesome .\\[7pt]

it was super dry and had a weird taste to the entire slice .\\
it was a great meal and the tacos were very kind of good .\\
it was super flavorful and had a nice texture of the whole side .\\[7pt]

%the cake portion was extremely light and a bit dry .\\
%the tortilla chips were perfectly cooked and spiced just delicious .\\
%the sandwich was warm and well a fresh .\\
\toprule
\multicolumn{1}{c}{From positive to negative}\\
\toprule
%amazing food , very unique .\\
%the food , flavorless and disappointing .\\
%disgusting food , very nasty and .\\[7pt]

%excellent service , good food , and nice atmosphere .\\
%the atmosphere , service , food was horrible .\\
%bland service , mediocre food , and no atmosphere whatsoever .\\[7pt]

i love the ladies here ! \\
i avoid all the time ! \\
i hate the doctor here ! \\[7pt]

%super friendly staff , great coffee , great soup and sandwiches .\\
%horrible , cold , slow , and not stuffy or angry . \\
%staff is slow, unfriendly, typical food, unimpressive food .\\[7pt]

my appetizer was also very good and unique .\\
my bf was n't too pleased with the beans .\\
my appetizer was also very cold and not fresh whatsoever .\\[7pt]

%excellent chinese and superb service .\\
%the quantity was horrible .\\
%bland chinese and superb service .\\[7pt]

came here with my wife and her grandmother !\\
came here with my wife and hated her !\\
came here with my wife and her son .\\

%go there , eat , enjoy !\\
%go there , avoid like crap !\\
%go there , eat , place for sure .\\[7pt]

%the husband did my nails and did a fantastic job . \\
%the owners were disrespectful and i did n't like it . \\
%the owner did my nails and did a sloppy job . \\[7pt]

%love the workers and the owner is always available and friendly .\\
%and the bathrooms are rude and the staff is not clue .\\
%unfortunately they are very unprofessional and they have always no issues .\\[7pt]

%even though the restaurant was packed i received my food very quickly !\\
%every time i was very disappointed and angry at the same time .\\
%if the food came in it was an hour and the waiter .\\[7pt]

%their food is a step above the local chinese restaurant . \\
%this is a low rating for the food and quality . \\
%their food is a step up in a bad neighborhood . \\[7pt]

%the new york eggrolls are outstanding and the beef dishes we ordered were flavorful .\\
%the beef and peppers were great , the fried chicken was just as good .\\
%the new bean burrito was bland and the sushi chefs were lacking quite flavorful .\\
\toprule
\end{tabular}
\end{center}
\caption{Sentiment transfer samples. The first line is an input sentence, the second and third lines are the generated sentences after sentiment transfer by \citet{hu2017controllable} and our cross-aligned auto-encoder, respectively.}\label{tab:examples}
\end{table}

\paragraph{Word substitution decipherment} Table~\ref{tab:synthetic} summarizes the performance of our model and the baselines on the decipherment task, at various levels of word substitution.  Consistent with our intuition, the last row in this table shows that the task is trivial when the parallel data is provided. In non-parallel case, the difficulty of the task is driven by the substitution rate. Across all the testing conditions, our cross-aligned model consistently outperforms its counterparts. The difference becomes more pronounced as the task becomes harder. When the substitution rate is 20\%, all methods do a reasonably good job in recovering substitutions. However,
when 100\% of the words are substituted (as expected in real
language decipherment), the poor performance of variational autoencoder and aligned auto-encoder rules out their application for this task.

%Tianxiao: on the upper line of the table substitute ratios with percentages, e.g. 0.3 --> 30%.
%also please report the parallel results for other ratios.
%if you want you can break up your substitution results by frequency

\begin{table}
\begin{center}
\begin{tabular}{l|ccccc|c}
\toprule
\multirow{2}{*}{Method} & \multicolumn{5}{c|}{Substitution decipher} & \multirow{2}{*}{Order recover}\\
 & 20\% & 40\% & 60\% & 80\% & 100\%\\
\toprule
No transfer (copy) & 56.4 & 21.4 & 6.3 & 4.5 & 0 & 5.1\\
Unigram matching& 74.3 & 48.1 & 17.8 & 10.7 & 1.2 & -\\
Variational auto-encoder & 79.8 & 59.6 & 44.6 & 34.4 & 0.9 & 5.3\\
Aligned auto-encoder & 81.0 & 68.9 & 50.7 & 45.6 & 7.2 & 5.2\\
Cross-aligned auto-encoder & \bf 83.8 & \bf 79.1 &\bf 74.7 &\bf 66.1 &\bf 57.4 & \bf 26.1\\
\hline
\hline
Parallel translation & 99.0 & 98.9 & 98.2 & 98.5 & 97.2 & 64.6\\
\toprule
\end{tabular}
\end{center}
\caption{Bleu scores of word substitution decipherment and word order recovery.}\label{tab:synthetic}
\end{table}

\paragraph{Word order recovery}

The last column in Table \ref{tab:synthetic} demonstrates the performance on the word order recovery task. Order recovery is much harder---even when trained with parallel data, the machine translation model achieves only 64.6 Bleu score. Note that some generated orderings may be completely valid (e.g., reordering conjunctions), but the models will be penalized for producing them. In this task, only the cross-aligned auto-encoder achieves grammatical reorder to a certain extent, demonstrated by its Bleu score 26.1. Other models fail this task, doing no better than no transfer. %Table~\ref{tab:order} shows examples of the reorderings generated by the system.
%Tianxiao: if you want, you can provide statistics how often the original sentence was recovered 

\iffalse
%Tianxiao: please give a few examples
\begin{table}
\begin{center}
\begin{tabular}{@{~~~}l@{~~~}|@{~~~}l@{~~~~~}}
\toprule
\multicolumn{2}{c}{Word order recovery}\\
\toprule
Ref & relax , check it out .\\
Input & it , relax . check out\\
Output & relax , check it out .\\[7pt]
Ref & cons : the pizza was not good .\\
Input & the pizza good was : not cons .\\
Output & the pizza was not good .\\[7pt]
Ref & they also have daily specials and ice cream which is really good .\\
Input & really is they have good which also ice and daily . specials cream\\
Output & really good they have daily specials and they have ice cream .\\
\toprule
\end{tabular}
\end{center}
\caption{Samples from our cross-aligned auto-encoder. The first line shows the original sentence, the second line the input after random permutation and the third line is the reordered sentence. 
%All words are lowercased to eliminate the information provided by first-word capitalization.
}\label{tab:order}
\end{table}
\fi

\section{Conclusion}
Transferring languages from one style to another has been previously trained using parallel data.
In this work, we formulate the task as \emph{a decipherment problem} with access only to non-parallel data.
The two data collections are assumed to be generated by a latent variable generative model.
Through this view, our method optimizes neural networks by forcing distributional alignment (invariance) over the latent space or sentence populations.
We demonstrate the effectiveness of our method on tasks that permit quantitative evaluation, such as sentiment transfer, word substitution decipherment and word ordering.
The decipherment view also provides an interesting open question---\emph{when can the joint distribution $p(\bm x_1, \bm x_2)$ be recovered given only marginal distributions?} 
%While we discussed the feasibility with a couple of simple distribution,
We believe addressing this general question would promote the style transfer research in both vision and NLP.

%In our generative framework, data $\bm x$'s distribution is specified by latent style $\bm y$'s distribution $p(\bm y)$, latent content $\bm z$'s distribution $p(\bm z)$, and the conditional distribution $p(\bm x|\bm y,\bm z)$. The task of style transfer necessarily differentiates the roles of $\bm y$ and $\bm z$ on $\bm x$. In this work we start from particular forms of $p(\bm x|\bm y,\bm z)$, and study the effects of $\bm z$'s different priors $p(\bm z)$.  While this is one way to approach recoverability, another way is to turn to $p(\bm x|\bm y,\bm z)$. Note that we can always assume $p(\bm z)$ has any prior, and rely on $p(\bm x|\bm y,\bm z)$ to approximate the data distribution. Then $\bm y$ and $\bm z$'s distinction should be reflected by $\bm x$'s dependency on $\bm y$ and $\bm z$ through $p(\bm x|\bm y,\bm z)$. This can be achieved by network's architecture design. For example, in language style transfer, we can still assume $\bm z$ has a prior $\mathcal N(\bm 0,\bm I)$, from which we first generate a ``content sequence'' through an RNN, then style $\bm y$'s effect is a simple attention which controls word order and a projection layer to vocabulary. What needs to be emphasized is that, if the generative model $p(\bm x|\bm y,\bm z)$ treats $\bm y$ and $\bm z$ on a par, yet also assume that $\bm z$ has a simple distribution with which can be easily tampered, then theoretical prerequisite for recoverability is violated and style transfer without content change is not guaranteed. 

\section*{Acknowledgments}
We thank Nicholas Matthews for helping to facilitate human evaluations, and Zhiting Hu for sharing his code. We also thank Jonas Mueller, Arjun Majumdar, Olga Simek, Danelle Shah, MIT NLP group and the reviewers for their helpful comments. This work was supported by MIT Lincoln Laboratory. 

\bibliography{ref}
\bibliographystyle{plainnat}

\newpage
\appendix
\section{Proof of Lemma \ref{lemma:gaussian}}
\gaussian*

\begin{proof}
$$p(\bm x|\bm y=(\bm A,\bm b))=\sum_{k=1}^K \pi_k\mathcal N(\bm x;\bm A\bm\mu_k+\bm b, \bm A\bm \Sigma_k\bm A^\top+\epsilon^2\bm I)$$

For different $\bm y=(\bm A,\bm b)$ and $\bm y'=(\bm A',\bm b')$, $p(\bm x|\bm y)=p(\bm x|\bm y')$ entails that for $k=1,\cdots,K$,
\begin{align*}
\begin{cases}
\bm A\bm\mu_k+\bm b = \bm A'\bm\mu_k+\bm b'\\
\bm A\bm \Sigma_k\bm A^\top = \bm A'\bm \Sigma_k\bm A'^\top
\end{cases}
\end{align*}
Since all $\mathcal Y$ are invertible,
$$(\bm A^{-1}\bm A')\bm\Sigma_k(\bm A^{-1}\bm A')^\top=\bm\Sigma_k$$ 
Suppose $\bm\Sigma_k=\bm Q_k\bm D_k\bm Q_k^\top$ is $\bm\Sigma_k$'s orthogonal diagonalization. If $k=1$, all solutions for $\bm A^{-1}\bm A'$ have the form:
$$\left\{\bm Q\bm D^{1/2}\bm U\bm D^{-1/2}\bm Q^\top \middle|\bm U\text{ is orthogonal}\right\}$$
However, when $K\ge 2$ and there are two different $\bm\Sigma_i\ne\bm\Sigma_j$, the only solution is $\bm A^{-1}\bm A'=\bm I$, i.e. $\bm A=\bm A'$, and thus $\bm b=\bm b'$.

Therefore, for all $\bm y\ne \bm y'$, $p(\bm x|\bm y)\ne p(\bm x|\bm y')$.
\end{proof}
\end{document}